\documentclass[letterpaper, 10 pt, conference]{ieeeconf}  

\IEEEoverridecommandlockouts                              

\usepackage{graphics} 
\usepackage{epsfig} 
\usepackage{times} 

\usepackage{amsmath} 
\usepackage{amssymb}  
\usepackage{pifont}
\newcommand{\cmark}{\ding{51}}%
\newcommand{\xmark}{\ding{55}}%
\usepackage{array}
\usepackage{amsthm}
\usepackage{booktabs}
\usepackage[dvipsnames]{xcolor}
\usepackage{siunitx}
\usepackage[hidelinks]{hyperref}
\hypersetup{
  colorlinks   = true, 
  urlcolor     = blue, 
  linkcolor    = black, 
  citecolor   = black 
}
\usepackage{amsfonts}
\newtheorem{definition}{Definition}
\newtheorem{problem}{Problem}
\newtheorem{theorem}{Theorem}

\usepackage{algorithm}
\usepackage[noend]{algpseudocode}
\algrenewcommand\algorithmicrequire{\textbf{Input:}}
\algrenewcommand\algorithmicensure{\textbf{Output:}}
\DeclareMathOperator*{\argmax}{arg\,max}

\usepackage{subcaption}
\usepackage{graphicx}
\usepackage{textcomp}
\usepackage{xcolor}
\usepackage{svg}
\def\BibTeX{{\rm B\kern-.05em{\sc i\kern-.025em b}\kern-.08em
    T\kern-.1667em\lower.7ex\hbox{E}\kern-.125emX}}

\newcommand{\bb}

\algdef{SE}[SUBALG]{Indent}{EndIndent}{}{\algorithmicend\ }%
\algtext*{Indent}
\algtext*{EndIndent}

\title{
Signal Temporal Logic Compliant Co-design of Planning and Control
}

\author{Manas Sashank Juvvi$^{\ddag, 1,2}$, Tushar Dilip Kurne$^{\ddag,1}$, Vaishnavi J$^{\ddag,1}$, Shishir Kolathaya$^1$, Pushpak Jagtap$^1$ 
\thanks{This work was supported in part by ARTPARK and Siemens.}
\thanks{$^\ddag$Authors contributed equally.}
\thanks{
$^1$ Centre for Cyber-Physical Systems, IISc, Bangalore, India} \thanks{$^2$ Technical University Delft, Delft {\tt\footnotesize \{tusharkurne, vaishnavij, shishirk,pushpak\}@iisc.ac.in,mjuvvi@tudelft.nl}}%
}

\begin{document}

\maketitle
\thispagestyle{empty}
\pagestyle{empty}

\begin{abstract}
This work proposes a novel co-design strategy that integrates trajectory planning and control for motion planning in autonomous robots, to handle signal temporal logic (STL) tasks. The approach is structured into two phases: $(i)$ learning spatio-temporal motion primitives to encapsulate the inherent robot-specific constraints and $(ii)$ constructing STL-compliant motion plan using learned spatio-temporal motion primitives. Initially, we employ reinforcement learning to construct a library of control policies that perform trajectories described by the motion primitives. Then, we map motion primitives to spatio-temporal characteristics. Subsequently, we present a sampling-based STL-compliant motion planning strategy tailored to meet the STL specification. We demonstrate the effectiveness and adaptability of our framework through experiments conducted on `Differential-drive robot' and `Quadruped' for multiple STL specifications and environments. Videos of the demonstrations can be accessed at \href{https://youtu.be/xo2cXRYdDPQ}{https://youtu.be/xo2cXRYdDPQ}.
The proposed framework is entirely model-free and capable of generating feasible STL-compliant motion plans across diverse environments.
\end{abstract}
\section{Introduction}\label{introduction}

Motion planning plays a significant role in autonomous navigation in complex environments for robots. Recent advancements in sampling-based methods have enhanced motion planning in complex spaces. Traditionally, the focus has been on developing controllers that guide robots from the initial state to the desired end state. However, as robots become more sophisticated, tasks grow in complexity, requiring advanced specification techniques. Signal temporal logic (STL) \cite{s1} offers an efficient framework for formulating complex tasks with spatial, temporal, and logical constraints. Using STL specifications, one can formally and efficiently capture complex motion planning tasks.  For example, a robot is given a task to transfer materials from a pick-up location to a drop location in a large workspace in a specified interval of time while avoiding static obstacles. In addition to this, it must avoid moving close to locations where some maintenance activity is going on for predetermined intervals of time, considering the locations as time-dependent obstacles. These tasks can be formally modeled using STL. Additionally, STL representation offers a comprehensive approach for quantitatively assessing system adherence to STL formulae through robustness metrics \cite{s2,s4}.

Several works in the literature have addressed STL specifications-based planning by incorporating STL semantics into the cost function of sampling-based path planners. In \cite{s5}, the authors developed a sampling-based method for synthesizing the controller that maximizes the spatial robustness of the STL specification using a partial robustness metric as the cost function as defined in \cite{s6}. The authors in \cite{s7} defined the cost functions to balance between spatial and time robustness of the sampling-based planner. Linard et al. \cite{s9} proposed an approach to deal with dynamic obstacles using a real-time Rapidly Exploring Random Tree (RT-RRT*) algorithm with an STL robustness-based cost function computed recursively along the nodes of the tree. 

While these works deal with generating path plans to satisfy STL specifications, most of them are not concerned with whether the plan is executable by the robot. The authors in \cite{s5} incorporate knowledge of the robot's mathematical model for path-planning and control purposes. Other studies explore tracking paths with temporal waypoints through optimization-based methods \cite{s10,s11}. However, all these approaches rely on knowledge of mathematical models of the system, which are difficult to derive for complex systems.

\begin{figure}[t]
    \centering
    \includegraphics[width=.36\textwidth, height=.18\textwidth]{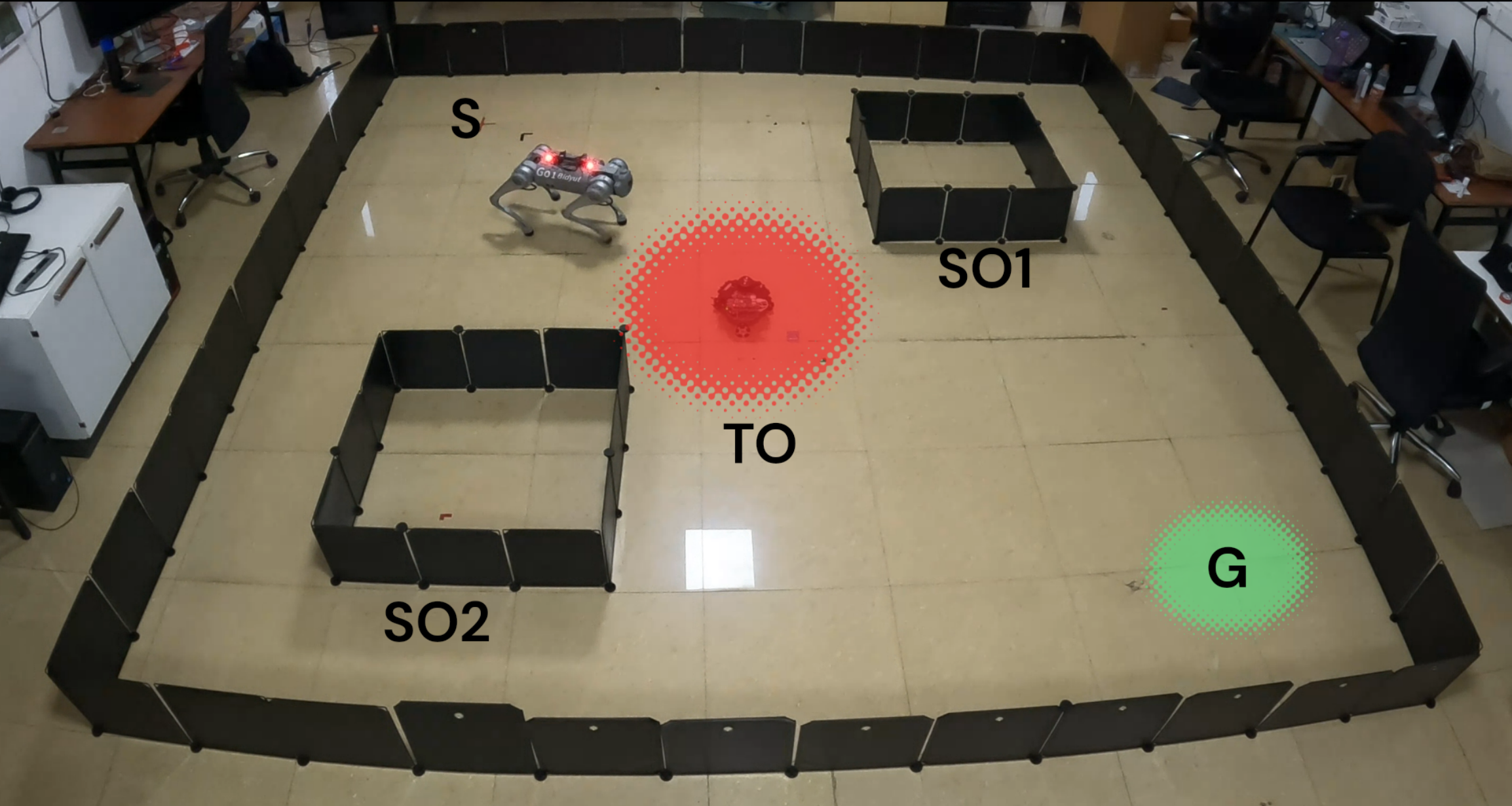}
    \caption{Experimental validation of the framework on Quadruped navigating in a workspace $W$ to reach goal region $G$ from starting position $S$, while avoiding static obstacles $SO1, SO2$ and timed obstacle $TO$ appearing between $[t_1,t_2]$
    }
    \label{fig:arena}
\end{figure}

To develop a control policy for a system with unknown dynamics, reinforcement learning (RL) emerges as a prominent approach in literature \cite{s13,s14}. In \cite{s15}, an RL agent is employed to develop a strategy enabling a robot to complete point-to-point tasks, followed by RRT for crafting a {feasible} path for the robot. In \cite{s16}, a hierarchical RL is used to build motion plans that follow the robot's dynamics and task constraints. In \cite{s20}, learning the complex motion is simplified by decomposing it into basic motion primitives. Motion primitives have been widely used to describe complex movements as a combination of sequential simple motions \cite{s21,s22}, they have been used in applications including UGVs, UAVs, and quadrupeds \cite{s20}, \cite{s24}.

To solve motion planning tasks represented using a limited class of STL specifications, \cite{s18} uses Q-learning to train the agent to achieve the assigned task. In \cite{s19}, the authors propose a funnel-based reward shaping in reinforcement learning algorithms to enforce a subset of STL specification in a tractable manner. However, a major drawback of these methods is the requirement of end-to-end retraining for any small changes in the environment or the task specifications.

In this paper, we address the motion planning problem for robotic systems with unknown dynamics subject to spatial, temporal, and logical constraints. In particular, the proposed approach involves a co-design strategy that integrates trajectory planning and control design to enforce STL task specifications while ensuring feasibility. To deal with unknown robot dynamics, we first choose a library of motion primitives (simple trajectories that the given robot can track) appropriate for the given robot, and use RL to learn policies that perform the motion primitives adhering to dynamics, state and input constraints. These policies are independent of the environment and specifications, eliminating the need to retrain them for different environments and specifications. To ensure that the robot's movement is accurate and it performs the given STL tasks, we need to sequentially apply the learned motion primitive policies for precise time durations. {We call the mapping between a robot's movement and time duration of policy application - a spatiotemporal relationship. We develop the reachability time estimators to obtain the spatiotemporal relationship by using a curve fitting method. Therefore, in order to obtain an STL-compliant path, we integrate learned policies and reachability estimators-based spatiotemporal information into the sampling-based trajectory planner. This planner not only creates the path nodes with position information but also determines the time component associated with each node using estimates from our trained reachability estimators}. Hence, a feasible path for the robot, along with the sequence of associated RL policies, will be produced to satisfy the given STL task. 
\section{Preliminaries}\label{preliminaries}
In this paper, we discuss the navigation problem for robots whose dynamics is assumed to be unknown within environments characterized by spatial, temporal, and logical constraints, described through signal temporal logic.
\subsection{Signal Temporal Logic}
\vspace{-0.2em}
 Signal temporal logic (STL) \cite{s1}  captures high-level specifications consisting of spatial, temporal, and logical constraints. Let $x: \mathbb{R}_{\geq 0}\mapsto X\subseteq\mathbb{R}^n$ be a continuous signal and a predicate function $h: X\mapsto \mathbb{R}$ such that predicate $\mu=\mathsf{true}$, if $h(x)\geq 0$ and $\mu=\mathsf{false}$, otherwise. The STL formula can be recursively obtained as
\begin{align}\label{stl_1}
    \phi &::=\mathsf{true} | \mu | \lnot \phi | \phi_1\vee\phi_2 | \phi_1 \wedge \phi_2 | \phi_1\mathcal{U}_{[t_1,t_2]}\phi_2,
\end{align}
where $\phi_1$ and $\phi_2$ are STL formulae of form $\phi$. Boolean operators for negation, disjunction, and conjunction are denoted by $\lnot, \vee$, and $\wedge$, respectively. 
The operator $\mathcal{U}$ represents the until operator in the interval $[t_1,t_2]\subset \mathbb{R}_{\geq 0}$ with $t_1 < t_2$. The temporal operators - eventually and always can be derived as follows: $\lozenge_{[t_1,t_2]}\phi =\mathsf{true}\mathcal{U}_{[t_1,t_2]}\phi, \square_{[t_1,t_2]}\phi = \lnot\lozenge_{[t_1,t_2]}\lnot \phi$, respectively. \\
\textbf{Robustness Metric} \cite{s2}: The satisfaction of STL formulae $\phi$ by a signal $x$ is quantified by its robustness metric $\rho^{\phi}(x,t)\in \mathbb{R}$. A signal $x$ is said to satisfy the STL specification $\phi$, denoted by $x\models\phi$, iff its robustness $\rho^\phi(x,0)\geq0$. Furthermore, robustness indicates the extent to which the specification is satisfied. For the above-mentioned STL formulae, robustness semantics can be defined as:
$\rho^\mu(x,t) = h(x(t))$, $
    \rho^{\lnot\phi}(x,t) = -\rho^\phi(x,t)$, 
    $\rho^{\phi_1\wedge \phi_2}(x,t) = \min(\rho^{\phi_1}(x,t)$, $\rho^{\phi_2}(x,t))$, $\rho^{\phi_1\vee \phi_2}( x, t) = \max(\rho^{\phi_1}(x,t), \rho^{\phi_2}(x,t))$,   
$\rho^{\phi_1\mathcal{U}_{[t_1,t_2]}\phi_2}(x,t)
= \underset{t'\in[t_1,t_2]}{\max} \hspace{-0.5em}\min\big(
\rho^{\phi_1}(x,t'),\hspace{-0.5em}\underset{t''\in[t_1,t']}{\min}\hspace{-0.3em}(\rho^{\phi_2}(x,t'')\hspace{-0.1em})
\hspace{-0.1em}\big)$,
$
    \rho^{\square_{[t_1,t_2]}\phi}(x,t)=\underset{t'\in[t_1,t_2]}{\min} \rho^{\phi}(x,t')$, $
    \rho^{\lozenge_{[t_1,t_2]}\phi}(x,t) =\underset{t'\in[t_1,t_2]}{\max} \rho^{\phi}(x,t')$. 
\subsection{Motion Primitive}
In this work, a motion primitive (MP) is defined as 
trajectory in state space $W$ that is dependent on time $t\in \mathbb{R}_{\geq0}$, $\mathbf{M}:\mathbb{R}_{\geq 0}\mapsto W\subset\mathbb{R}^n$. Every motion primitive outputs a state from the given state space for every time instant, thus forming a simple trajectory in the space.
Decomposing and reconstructing any complex trajectory from these basic motion primitives creates a robust framework ensuring adaptability in diverse environments. As each MP corresponds to a trajectory {trackable} by the robot within its dynamic constraints, any path derived by sequentially connecting these motion primitives will ensure the feasibility of navigating the generated path. Also, learning control actions to track motion primitives instead of learning dynamics significantly reduces the training time and computation requirements, presenting a scalable solution that is adaptable to various robots.  Generating MPs for various robotic systems is not discussed in the scope of this paper as there are several ways in literature to generate MPs such as splines \cite{s33}, dynamic motion primitives \cite{s36}, probabilistic motion primitive (ProMP) \cite{s37}. Generating MPs for robotic systems such as UGVs, UAVs, and Quadrupeds is discussed in \cite{s20}, \cite{s24}.
\vspace{-0.3em}
\subsection{Reinforcement Learning} 
\vspace{-0.2em}
To create policies that result in the execution of the chosen library of motion primitives for a robot with unknown dynamics, we use reinforcement learning (RL) \cite{s13}. RL maximizes a reward signal by mapping environment states (or observations) to actions through interaction. The learning agent tries to achieve its goal by taking specific actions and sensing the change of states and their corresponding rewards. Then, it maximizes its long-term reward by selecting an optimal control policy \cite{s13}.\\
\textit{Markov Decision Process}: 
Markov decision process (MDP) is a model for sequential stochastic decision problems. It is defined by a tuple $M = (S, A, P_a, R, \gamma )$, where $S\subset \mathbb{R}^n$ is the state space, $A\subset\mathbb{R}^m$ is the action space, $P_a(s,s’) = Pr(s_{t+1} = s’|s_t = s, a_t = a)$ is the probability that the state at time $t+1$ is $s’$, given that the current state at time $t$ is $s$, and the action taken is $a$. $R(s,a)$ is the immediate reward received when action $a$ is taken in state $s$. $\gamma \in [0,1]$ is the discount factor that determines the present value of future rewards. A policy is a mapping from states to actions that specifies the action to take at a particular state. An optimal policy is one that maximizes the total discounted return while performing a task. The total discounted return is given by: $G_t = \sum_{i=0}^{\infty}\gamma^{i}R_{t+i}$, where $R_{t+i}$ is the reward received at time $t+i$.  There are various algorithms for computing an RL policy for the continuous action space. Examples include soft-actor critic \cite{sac1} and proximal policy optimization \cite{ppo}. 
\vspace{-0.3em}
\subsection{Reachability Time Estimator}
\vspace{-0.2em}
Any control policy that is used to achieve a desired displacement in the robot's spatial position has to be applied for a precise amount of time. The spatiotemporal relationship gives this information.  
\begin{definition}[Spatiotemporal relationship] 
We define a spatiotemporal relationship $g_\pi$ as the relationship between the displacement of a robot $d$ when a policy $\pi$ is applied and the time horizon $t_\pi$ for which it is applied.
    \begin{align*}
        t_\pi = g_\pi(d), 
    \end{align*}
    where $g_\pi$ is a function that maps the distance covered to time when policy $\pi$ is implemented.
\end{definition}

Since this spatiotemporal relationship is unknown, we need to create a mapping function that connects the spatial and temporal features of these policies. 


\section{Problem Statement}\label{problem_statement}
\vspace{-0.1em}
 Let $\mathcal{Q}\subset\mathbb{R}^m\times\mathbb{R}_{\geq 0}$ be a configuration space defined as $\mathcal{Q} = W \times T$, where $W\subset\mathbb{R}^m$ is a $m$ dimensional workspace and $T\subset\mathbb{R}_{\geq 0}$ indicates time interval. A tree in the configuration space $\mathcal{Q}$ is represented by a graph $G = (\mathcal{N},E)$, where set of nodes is denoted by $\mathcal{N} \subset \mathcal{Q}$ and the $i^{th}$ node is represented by $n_i=(x_{1_i},x_{2_i},...,x_{m_i},t_i)\in \mathcal{N}$, where $(x_1,x_2,..,x_m)\in W$ and $t_i\in T $, and $E$ is the set of edges, $e_{ij}=(n_i,n_j) \in E$ connecting two nodes $n_i, n_j$ in the graph. The initial configuration of the robot is known in $\mathcal{Q}$, and it will be considered as the root node of the tree $n_0 = (x_{1_0},x_{2_0},..., x_{m_0},t_0)$. For $n_k\in\mathcal N$, the trajectory from root node $n_0$ till node $n_k\in\mathcal{N}$ is defined by $\mathcal{T}_{n_k}=(n_0,n_1,\ldots,n_k)$ such that $(n_i,n_{i+1})\in E$ for all $i\in\{0,\ldots,k-1\}$.

The aim of this work is to find a motion plan or trajectory along with an associated control policy that ensures the given STL specification $\phi$ with maximum robustness. Let $\mathcal{T} = \{\mathcal{T}_{n_k} \vert {\rho^{\phi}(\mathcal{T}_{n_k})}\geq 0\}$
is the set of all feasible trajectories for an STL specification $\phi$,
where $\mathcal{T}_{n_k}\models\phi$
and $\overline{\mathcal{T}}$ be the trajectory with maximum robustness defined as $\overline{\mathcal{T}}=\argmax_{\mathcal{T}_{n_k}\in \mathcal{T}}(\rho^\phi(\mathcal{T}_{n_k}))$. 
\begin{problem}\label{problem}
Given a robotic system $\mathcal{R}$ with unknown dynamics operating in workspace $W$, and an STL specification $\phi$ of the form \eqref{stl_1}, develop a control strategy and obtain a robust and feasible trajectory $\overline{\mathcal{T}}$ to satisfy $\phi$. 
\end{problem}
We address \text{Problem 1} with a framework 
consisting of two main components: $(i)$ Learning spatiotemporal motion primitives and $(ii)$ STL-compliant motion planning, elaborated in the following section.
\section{Proposed Framework}\label{Methodology}
\vspace{-0.2em}
This section presents the proposed STL-compliant motion planning framework solving Problem \ref{problem}. 

\subsection{Learning Spatiotemporal Motion Primitives}
\subsubsection{Generating Motion Primitives Library}
Consider a set of $N$ motion primitive classes, $\mathcal{M} = \{\mathbf{M}_1, \mathbf{M}_2, ..., \mathbf{M}_N\}$, each $\mathbf{M}_i$ is a collection of motion primitives characterizing a simple trajectory in the workspace. For example, a set of MPs {represents} straight-line trajectories tracked with distinct maximum velocities that the robot can attain. Such a set of motion primitives forms a motion primitive class. 

Therefore, to track a trajectory with certain velocities, distinct policies can be designed for each $\mathbf{M}_N$ in $\mathcal{M}$. 
In our work, we develop a control policy set $\xi_{\mathbf{M}_i} =\{\pi_{i1}, \pi_{i2},..., \pi_{i{p_i}}\}$ corresponding to the class $\mathbf{M}_i$, that ensures that the MPs are tracked. Here, $p_i$ represents the number of policies developed for $\mathbf{M}_i$, and $\pi_{ij}$, $j\in \{1,...,p_i\}$ are the control policies associated with class $\mathbf{M}_i$ and characterized by a specific velocity (Line 2 of Algorithm \ref{algorithm}).

The set of all controllers $\xi=\bigcup_{i=1}^{N}\xi_{\mathbf{M}_i}$ will be utilized to synthesize the control strategy (line 6 of Algorithm \ref{algorithm}). Given a system with unknown dynamics, we leverage RL for policy development. RL training involves a state set $S\subset \mathbb{R}^n$, an action space $A\subset\mathbb{R}^m$ and 
the reward structure $R_{ij}: S\times A\mapsto \mathbb{R}$ that helps to guide the system to perform the corresponding motion primitive by learning RL policy $\pi_{ij}$ and penalizes deviations from it. Upon completing the RL training for all motion primitives, we obtain the policy set $\xi$.

\subsubsection{Reachability Time Estimator}
The amount of robot displacement, whenever the generated RL policies $\xi$ are applied for different time intervals, varies. 
To meet the STL specification for a robotic system, we design reachability time estimators that find the spatiotemporal relationship $g_\pi(.)$ that our current RL policies do not capture.
Due to the non-linearities in system dynamics and the environment, function $g_\pi$ can be non-linear. To address this, we propose using a curve-fitting estimator for each policy in the set $\xi$. These estimators map the relative change in pose to the time horizon for which policies need to be applied. 

We apply each policy to the robot over an increasing time horizon, for example, 0.1, 0.2,...,  50 seconds. Then, we record the robot's pose changes corresponding to the time horizons. This data is used to train the curve-fitting model $g_{ij}$, with a change in the pose as input and a time horizon as output (Algorithm \ref{algorithm}, lines 3-5).
The final set of reachability estimators is $\mathcal{G}=\cup_{i=1}^N\cup_{j=1}^{p_i}g_{ij}$. These estimators, along with the motion primitive library, capture spatial and temporal data of the control policies for the robot.

\subsection{STL-compliant Motion Planner}
In this subsection, we propose a sampling-based motion planning algorithm built upon \cite{rrt*} that incrementally constructs a tree exploring the configuration space $\mathcal{Q} = W \times T$ to fulfil a specified STL specification $\phi$ in \eqref{stl_1}. We utilize the set of control policies $\xi$ and the spatio-temporal functions $f$ in the reachability time estimator $\mathcal{G}$ to generate feasible paths for a given robot that satisfy $\phi$. 

Next, we explain Algorithm \ref{path_planner} briefly. We start by fixing the root of the tree at the robot's initial state and initialize the set of the nodes $\mathcal{N}$ with $n_0 = (x_{1_0},x_{2_0},..., x_{m_0},t_0)$, and assign the cost of the $n_k$th node as
\begin{equation}
    J_{\phi}(n_{k}) = \sum_{n=n_{0}}^{n_{k}} - {\rho}^{\phi}(\mathcal{T}_n). \label{cost_fn}
\end{equation}

{The trajectory joining the root node $n_0$ to node $n_k$ is discrete $\mathcal{T}_{n_k}$, and the robustness $\rho^\phi(\mathcal{T}_{n_k})$ of the trajectory is derived using the robustness measure as described in \cite{s9}}.
 
The tree $G = (\mathcal{N},E)$ grows by adding nodes and corresponding edges to sets $\mathcal{N}$ and $E$ respectively in each iteration $i$ (see Algorithm \ref{path_planner}).
Let $w_{i} = (x_{1_i},x_{2_i},...,x_{m_i})$, we use uniform random sampling to sample a new point $w_{rnd}$ from $W$. The previously generated RL policies are trained to move the robot for a minimum distance $d_{min}$, and a maximum distance $d_{max}$. Additionally, the reachability time estimators are trained with data obtained after applying these policies. Therefore, the time horizon that the policies have to be applied for, can be predicted with high accuracy any distance between $[d_{min}, d_{max}]$.
Let the nearest point of $w_{rnd}$ be $w_{nearest}$ in $W$. $w_{rnd}$ is steered towards $w_{nearest}$ such that its distance from $w_{nearest}$ is between $d_{min}$ and $d_{max}$. This will ensure the learned policies can track the path between these two nodes. The extended position of the $w_{rnd}$ is updated to $ w_{new}$ (lines 3-5). We now need the temporal component $t_{new}$ (the time it takes to reach from $n_0$ to $n_{new}$) to build the new node $n_{new}$. Let the corresponding node for $w_{nearest}$ in $Q$ be $n_{parent}$, which is temporarily assigned as the parent to $n_{new}$. We randomly select a sequence of policies $\mathcal{P}$ from $\xi$ that enable the robot to reach $w_{new}$ from $w_{nearest}$, and predict the time required ($t_{\mathcal{P}} = \sum_{\mu\in \mathcal{P}} g_\mu(dist(w_{nearest}, w_{new})$) to travel from $w_{nearest}$ to $w_{new}$ using reachability estimators from the set $\mathcal{G}$. Then, time taken to reach $n_{new}$ from root node $n_0$ is the sum of time taken to reach $w_{nearest}$ and $t_\mathcal{P}$, i.e., $t_{new} = t_{parent}+t_\mathcal{P}$. Once the sequence of policies, and $t_{new}$ is computed, we check if the edge joining the $n_{parent}$ to $n_{new}$ is collision-free in $\mathcal{Q}$. If the edge is collision-free, the new node $n_{new}$ is appended to $\mathcal{N}$ and the edge $e_{parent,new} = (n_{parent},n_{new})$ to $E$ (lines 6 to 12), and update the tree $G$. The cost of $n_{new}$ is calculated according to \eqref{cost_fn}, which ensures that the path generated maximizes the robustness of $\phi$. Each edge will have a corresponding sequence of policies applied to progress from parent to child node, along with the duration of application. This structured approach ensures that our plan meets $\phi$ and adheres to the dynamic constraints of the robot.
To optimize the constructed tree, the tree is rewired similar to \cite{rrt*}. Upon adding a new node to the tree, we analyze its neighbouring nodes within $[d_{min},d_{max}]$ range. For each neighbouring node, we evaluate whether adopting the $n_{new}$ as a new parent would result in a lower overall cost from the root node. If a more cost-effective and collision-free path exists, we rewire the parent-child relationships within the tree accordingly (Lines 13-18). 

Finally, after the tree-building process is complete, the trajectories are examined based on their robustness, and the trajectory with the highest robustness, $\overline{\mathcal{T}}$ is returned along with the sequence of policies associated with it (lines 19-23). 
In the tree-building's early stages, before a full path is built, robustness is assessed based on only the currently relevant segments of the STL specification that are active in time. When a full path is generated, robustness evaluation extends across the entire STL specification. Initially, this strategy yields a spatio-temporally feasible path, which is then optimized to maximize the robustness throughout the entire signal by repeated sampling of nodes. The proposed solution is probabilistically complete as the number of sampled nodes tends to infinity.

\begin{algorithm}[t]
\caption{Algorithm for STL-Compliant framework}
\label{algorithm}
\begin{algorithmic}[1]
    \Require $\mathcal{R}$, $W$, $\phi$, $\mathcal{M}$
    \For {$\mathbf{M}_i$ in $\mathcal{M}$}
        \State $\xi_{\mathbf{M}_i} \leftarrow $\textsc{GeneratePolicies}$(\mathcal{R}, W, \mathbf{M}_i)$
        \For{$\pi_{ij} $ in $\xi_{\mathbf{M}_i}$}
            \State Obtain sampled dataset $S_{ij} $
            \State $g_{ij}\leftarrow $ \textsc{ReachabilityEstimator}$(S_{ij})$
        \EndFor
    \EndFor
    \State $\xi = \cup_{i=1}^{N}\cup_{j=1}^{p_i}\pi_{ij}$
    \State $\mathcal{G} = \cup_{i=1}^{N}\cup_{j=1}^{p_i}g_{ij}$
    \State $d_{min},d_{max} \leftarrow$ \textsc{ComputeDistanceRange}$(\xi, \mathcal{G})$
    \State $\overline{\mathcal{T}}, \mathcal{P} \leftarrow $\textsc{STLTrajectoryPlanner}$(W, \phi, \mathcal{G}, d_{min},d_{max})$
    \For{$node$ in $\overline{\mathcal{T}}$}
        \State \textsc{MotionPlanExecution}$(\mathcal{R}, W, node,\xi, \mathcal{P},\mathcal{G})$
        \If{\textsc{ErrorInTracking} is True}
            \State $\overline{\mathcal{T}}, \mathcal{P} \leftarrow $ \textsc{STLTrajectoryPlanner}$(W, \phi, \mathcal{G}, d_{min},d_{max})$
        \EndIf
    \EndFor
\end{algorithmic}
  \end{algorithm}

\subsection{Execution of the Motion Plan}
\vspace{-0.3em}
In order to execute the generated motion plan, the robot traverses from the current position at node $n_i$ to the next node $n_{i+1}$ of $\overline{\mathcal{T}}$ within $\Delta t=t_{i+1}-t_i$, using the sequence of policies ($\mathcal{P}$) returned by the motion planner. We ensure that small errors in tracking the planned path is resolved using policies from the set $\xi$ in closed-loop control.
However, if the robot deviates significantly from the preplanned path during execution, the path is re-planned, accounting for unforeseen uncertainties such as changes in the environment, etc. For the replanning phase, the robot's current position, and current time instant forms the root node of the new tree. Furthermore, we design the new tree as per Algorithm \ref{path_planner}, explained in Section \ref{Methodology}(B).
Consequently, we update the control strategy to track the new planned path (Line 10-13 of Algorithm \ref{algorithm}).

\begin{theorem}
    {Given a robotic system $R$ in a workspace and STL specification $\phi$, if the proposed STL trajectory planner generates a trajectory $(\overline{\mathcal{T}})$ and policy sequence $(\mathcal{P})$ (line 9, 13 of Algorithm \ref{algorithm}), then it ensures feasibility and STL satisfaction.}
\end{theorem}
\begin{proof}
     RL policies $\xi$ and reachability estimators $\mathcal{G}$ are designed for the given robotic system. The node sampling in trajectory generation is based on the minimum ($d_{min}$) and maximum ($d_{max}$) distance that the generated policies can cover. Furthermore, the temporal component of each node is assigned according to the selected policy $\pi_{ij}\in\xi$ and the corresponding reachability estimator $g_{ij}\in\mathcal{G}$. This ensures the feasibility of the trajectory and the policy sequence. Finally, the trajectory with maximum positive robustness ($\overline{\mathcal{T}}$, Line 22 of Algorithm \ref{path_planner}) is selected, ensuring that the STL specification is satisfied. 
\end{proof}
\vspace{-0.3em}
\begin{algorithm}
\caption{STLTrajectoryPlanner($W, \phi, \xi, \mathcal{G}, d_{min}, d_{max}  $)}
\label{path_planner}
\begin{algorithmic}[1]
    \Require $W$, $\phi$, $\xi$, $\mathcal{G}, $$d_{min}$, $d_{max}$
    \Ensure $\overline{\mathcal{T}}, \mathcal{P}$
    \State $\mathcal{N}\leftarrow \{n_0\}, E\leftarrow\{\}, \mathcal{T}\leftarrow\{\}, \mathcal{P}\leftarrow \{\}$
    \For{$i\leq max\_iteration$}
        \State $w_{rnd}\leftarrow $ Sample $W_{free}$
        \State $w_{nearest}\leftarrow \textsc{Nearest}$$(G^W, w_{rnd})$
        \State $w_{new}\leftarrow \textsc{Steer}$$(w_{rnd}, w_{nearest}, [d_{min},d_{max}])$
        \State $parent(n_{new})\leftarrow n_{parent}$
        \State $t_{new}, temp\leftarrow \textsc{SamplePolicies}(\mathcal{G}, n_{parent},n_{new})$
        \State $n_{new}\leftarrow\textsc{Assign}(t_{new}, w_{new})$
        \If{$\textsc{CollisionFree}(n_{parent},n_{new})$}
            \State $\mathcal{N}\leftarrow\mathcal{N}\cup \{n_{new}\}$
            \State $E\leftarrow E\cup \{(n_{parent},n_{new})\}$
            \State $\mathcal{P}((n_{parent},n_{new}))\leftarrow temp$
        \EndIf
            \For{ each $n_{near}\in\mathcal{N}_{kNear}$}
                \State $t_{near}\leftarrow\textsc{SamplePolicies}(\mathcal{G}, n_{parent},n_{new})$
                \If{$\textsc{CollisionFree}(n_{near},n_{new})$ and $(J_\phi(n_{new}) < J_\phi(parent(n_{near})))$}
                \State $E\leftarrow E\setminus\{parent(n_{near}), n_{near}\}\cup\{n_{new},n_{near}\}$
                \State $\mathcal{P}((n_{new},n_{near}))\leftarrow temp$
                \EndIf
            \EndFor
            
    \State $i\leftarrow i+1$
    \EndFor
    \For{$n_k\in \mathcal{N}$}
        \If{$\rho^\phi(\mathcal{T}_{n_k})\geq 0$}
            \State $\mathcal{T}=\mathcal{T}\cup\{\mathcal{T}_{n_k}\}$
        \EndIf
    \EndFor
    \State $\overline{\mathcal{T}}= \argmax_{\mathcal{T}_{n_k}\in \mathcal{T}}(\rho^\phi(\mathcal{T}_{n_k}))$
    \State \textbf{return} $\overline{\mathcal{T}}, \mathcal{P}$
\end{algorithmic}
\end{algorithm}

\vspace{-0.2em}
\section{Results And Discussion}\label{results}
We validate our framework on differential drive mobile robots and quadruped across various environments. Training and simulations are conducted on a computer equipped with an AMD Ryzen 9 5950X, 128 GB DDR5 RAM, and NVIDIA RTX 3060Ti 32 GB GDDR6. The experimental system operates on Ubuntu 20.04 with ROS Noetic distribution.

\vspace{-0.3em}
\subsection{Case study 1 - Turtlebot3}
We validated the proposed framework by conducting a real-world experiment on the Turtlebot3 (a differential drive robot) that has to navigate in a 2D workspace $W=[0,3]\times[0,3]$ to enforce the following STL specification:
\begin{align}\label{task_turtlebot1}
    \phi_1 &=  \Box_{[0,50]} ([0,0]\leq\mathbf{x}\leq[3,3]) \wedge \lozenge_{[20,25]}(\|\mathbf{x}-\mathbf{g_1}\|_2\leq \nonumber\\0.25) 
    &\wedge \lozenge_{[45,50]}(\|\mathbf{x}-\mathbf{g_2}\|_2\leq 0.25)\wedge \Box_{[0, 30]} (\|\mathbf{x}-\mathbf{o_t}\|_2>\nonumber\\0.5)  
    & \wedge\Box_{[0,55]} ( \|\mathbf{x} - \mathbf{o_i}\|_2 > 0.4), \forall i=\{1,2\},
\end{align}
where $\mathbf{x}=(x,y)$ is the robot's position.
To meet the STL specification $\phi_1$, starting anywhere in $W$ at time $t=0$, Turtlebot3 should eventually reach $\mathbf{g_1}=(2.5,0.5)$ between 20-25 seconds, and then reach $\mathbf{g_2}=(0.5,2.5)$ between 45-50 seconds, while always maintaining a safe distance from the static obstacles at $\mathbf{o_1}=(1.5,0.3)$ and $\mathbf{o_2}=(1.5,2.4)$, the timed obstacle TO that appears between 0- 30 seconds interval centred at $\mathbf{o_t}=(1.5,1.5)$, and the robot must always remain inside $W$ for 50 seconds. The environment is shown in Fig. \ref{turtlebot3}(a)-(b).
We consider that Turtlebot3 can perform two fundamental motions: translational and rotational. The set of motion primitives classes is defined as $\mathcal{M} = \{\mathbf{M}_1, \mathbf{M}_2, \mathbf{M}_3, \mathbf{M}_4\} = \{Clockwise, Counterclockwise, Forward, Backward\}$. We then learned RL policies $\xi_{\mathbf{M}_i}$, considering that each motion can be attained at various velocities. The set of RL policies $\xi = \bigcup_{i=1}^4\xi_{\mathbf{M}_i}$ are trained that corresponds to $\mathcal{M}$. 

\textit{Training RL Policies:} To learn RL policies (Section \ref{preliminaries}), we discuss the procedure for rotational motion primitives $\mathbf{M}_1$ and $\mathbf{M}_2$. For rotational motion primitives, we consider state space $S$ with states $\Delta d$ and $\omega$ representing the displacement of the robot from its initial position and the angular velocity of the robot, respectively, and $v_{l}$ and $v_{r}$ are robot's left and right wheel velocities (in radians per second) representing actions. The reward function is formulated to ensure the robot rotates in the correct direction with angular velocity $|\omega| \leq |\omega_{max_{ij}}|$ for $i^{th}$ MP and $j^{th}$ policy, and penalizes any displacement from its initial position as defined in \eqref{rew_rotn}. 
    \begin{figure*}[t!]
  \begin{subfigure}{.23\linewidth}
    \includegraphics[width=\linewidth]{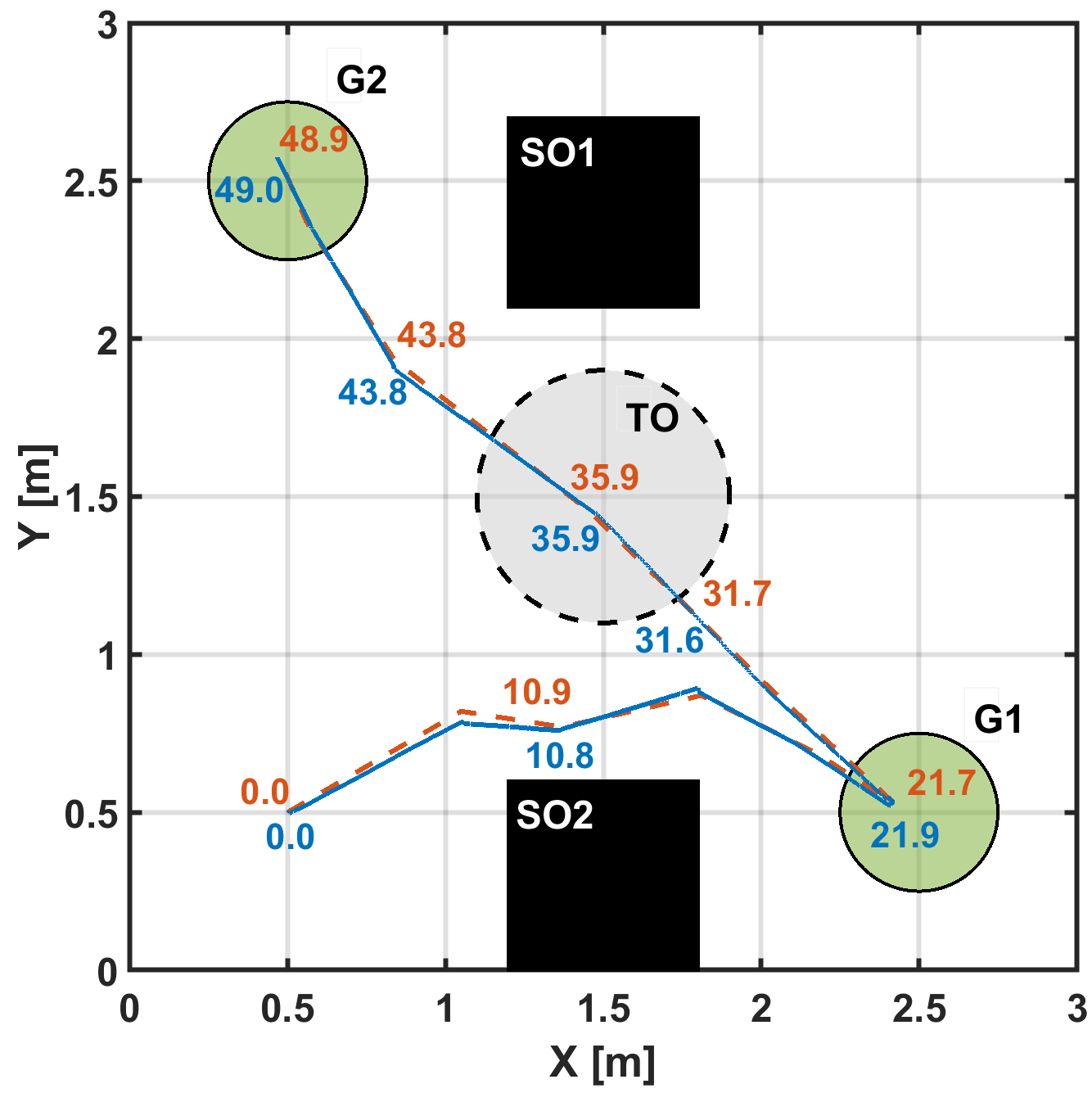}
    \caption{}%
  \end{subfigure}\hfil
  \begin{subfigure}{.23\linewidth}
    \includegraphics[width=\linewidth]{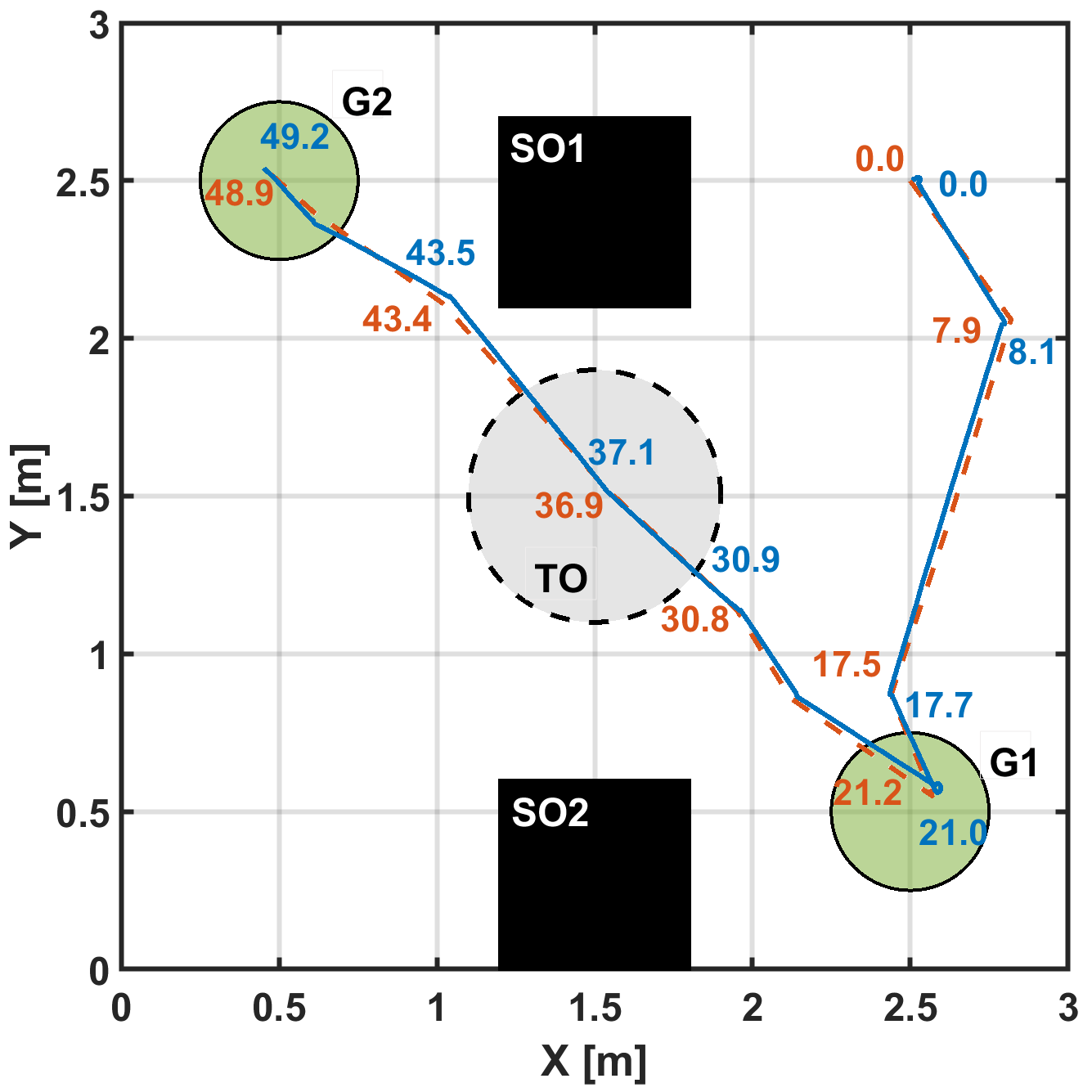}
    \caption{}%
  \end{subfigure}\hfil
   \begin{subfigure}{.24\linewidth}
    \includegraphics[width=\linewidth]{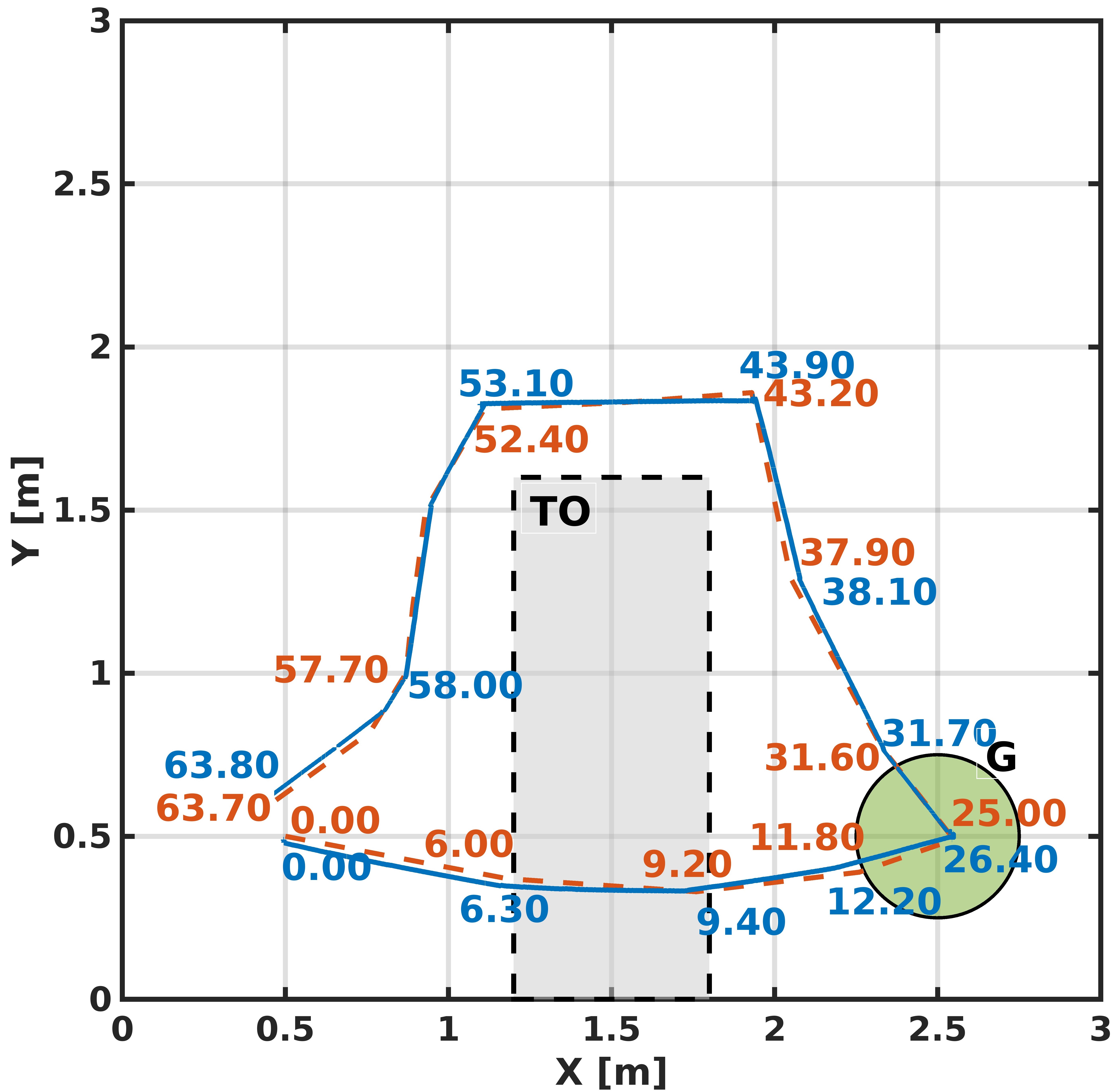}
    \caption{}%
  \end{subfigure}\hfil%
  \begin{subfigure}{.24\linewidth}
    \includegraphics[width=\linewidth]{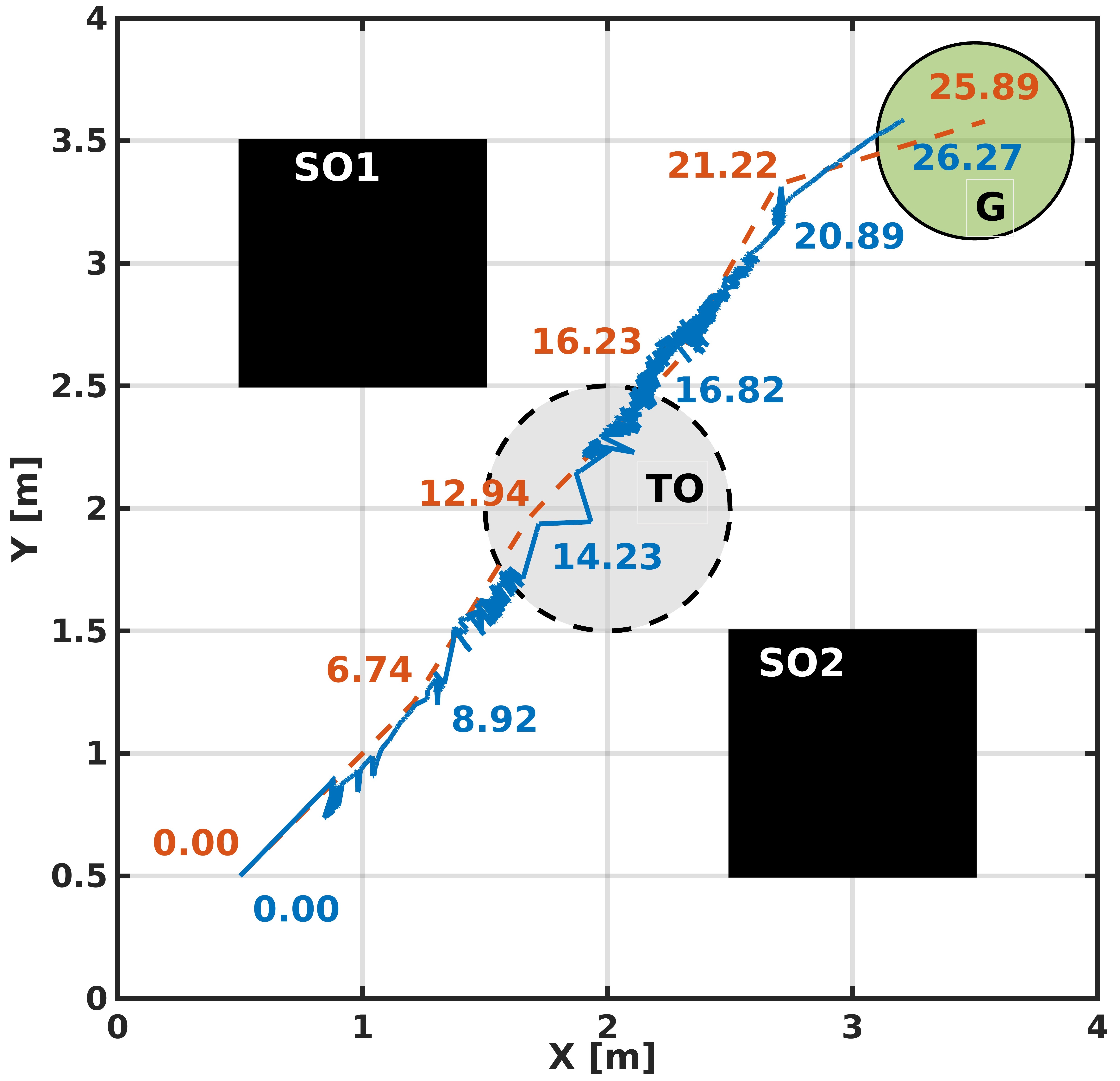}%
    \caption{}%
  \end{subfigure}\hfil
    \caption{ Results for Case study - Goal Regions - $G,G1,G2$, Static obstacles - $SO1, SO2$, Time-dependent obstacle - $TO$, Planned (Orange) and actual (Blue) trajectory with their respective data points representing the time required (in seconds) to reach the node. 
    (a) Planned and actual trajectory satisfy $\phi_1$ with Turtlebot3's starting position at $(0.5,0.5)$, (b) Planned and actual trajectory satisfies $\phi_1$ with Turtlebot3's starting position at $(2.5,2.5)$, (c) Planned and actual trajectory satisfy $\phi_2$ with Turtlebot3's starting position at $(0.5,0.5)$, (d) Quadruped - Planned and actual trajectory satisfy $\phi_3$ starting at $(0.5,0.5)$.}
    \label{turtlebot3}
\end{figure*}

Similarly, for linear motion primitives $\mathbf{M}_3$ and $\mathbf{M}_4$, the reward functions are designed to ensure that (i) the robot moves linearly in the desired direction with velocity $|V| \leq |V_{b_{max}}|$, and (ii) lateral deviation ($ld>\qty{0.01}{\m}$) is penalized. 

By selecting different values of $V_{b_{max}}$ and $\omega_{max}$, in the range of  $[\qty[per-mode = symbol]{-0.22}{\m\per\s}, \qty[per-mode = symbol]{0.22}{\m\per\s}]$, and $[\qty[per-mode = symbol]{-2.84}{\radian\per\s}, \qty[per-mode = symbol]{2.84}{\radian\per\s}]$, 
respectively, we form the set of policies $\xi$ that move the robot in forward, backward, clockwise and counterclockwise. 
\begin{align} \label{rew_rotn}
        r_{1_{ij}} &= \begin{cases}
         k_1*\omega, &\text{if } |\omega| \leq |\omega_{max_{ij}}|\\
         0, &\text{otherwise}
     \end{cases},   \qquad
     r_2 = -k_2*\Delta d \nonumber\\
     R_{ij} &= r_{1_{ij}} + r_2, \text{ for } i\in \{1,2\}, j\in\{1,...,p_i\},
    \end{align}
    where $k_1<0$ for $\mathbf{M}_1$, $k_1>0$ for $\mathbf{M}_2$, $k_2 > 0$, and $p_1,p_2\in\mathbb{N}$ are the number of policies that rotate the robot in a clockwise and counterclockwise manner, respectively.
The policies were trained using the PPO algorithm \cite{ppo}. Next, the reachability estimator models $\mathcal{G}$ for the control policies $\pi\in\xi$ that map the displacement (linear $\Delta d$, and angular $\Delta\theta$) to the time horizon are obtained. Data collection and training of the reachability estimator is carried out as described in Section \ref{Methodology}. Once the trajectory and the policy sequence is returned by Alg. \ref{path_planner}, we apply the policies from $\xi$ and track $\overline{\mathcal{T}}$ until Turtlebot3 satisfies \eqref{task_turtlebot1}. From Fig \ref{turtlebot3}(a) and (b), we can observe the control strategy enforced the satisfaction of $\phi_1$ with different starting positions of the robot.

In order to validate the adaptability of the proposed framework using the same policy set and reachability estimator set, we consider solving a different STL task in a new environment. The STL specification is given by
$ \phi_2 =  \Box_{[0,65]} ([0,0]\leq\mathbf{x}\leq[3,3]) \wedge \square_{[20,25]}(\|\mathbf{x}-\mathbf{g_1}\|_2\leq 0.25)\wedge \lozenge_{[30,65]}(\|\mathbf{x}-\mathbf{x_0}\|_2\leq 0.25)\wedge \Box_{[30, 65]} ((x-1.2)(x-1.8)>0 \vee y(y-1.6)>0)$, where, $\mathbf{x_0}$ is robot's initial position, $\mathbf{g_1}=(2.5, 0.5)$.
The planned and executed trajectory adheres to specification $\phi_2$, reaching $\mathbf{g1}$  and returning safely to its starting point within 30 to 65 seconds Fig. \ref{turtlebot3}(c). Generating the motion plan took 4.79 seconds on an average for both initial conditions of $\phi_1$ and 3.64 seconds for $\phi_2$.

\vspace{-0.5em}
\subsection{Case Study 2 - Quadruped}
\vspace{-0.3em}
We demonstrate the proposed framework on a Unitree Go1 quadruped in a workspace $W=[0,4]\times[0,4]$, with goal and obstacle locations as shown in Fig \ref{fig:arena}. The framework is demonstrated for the following specifications: 
    $\phi_3 = \Box_{[0,30]} ([0,0]\leq\mathbf{x}\leq[4,4]) \wedge\lozenge_{[25,30]}(\|\mathbf{x}-\mathbf{g_1}\|_2\leq 0.4)\wedge\Box_{[0, 10]} (\|\mathbf{x}-\mathbf{o_t}\|_2>0.5) \wedge\Box_{[0,30]} ( \|\mathbf{x}- \mathbf{o_i}\|_2 > 0.4), \forall i=\{1,2\}$,
where, $\mathbf{g_1}=(3.5, 3.5),\mathbf{o_1}=(1,3), \mathbf{o_2}=(3,1), \mathbf{o_t}=(2,2)$. The set of motion primitives classes, in this case, is similar to the Turtlebot3 case - $\{Clockwise, Counterclockwise, Forward, Backward\}$. The RL policies are trained for this quadruped that executes these motions. 
Fig. \ref{turtlebot3}(d) shows the planned trajectory and actual trajectory (which required re-planning due to the sim2real transfer gap) in the workspace satisfies $\phi_3$. 
The video of the implementation can be found at \href{https://youtu.be/xo2cXRYdDPQ}{https://youtu.be/xo2cXRYdDPQ}.

\section{Discussion and Conclusion}
We have summarized the advantages of our framework compared to other relevant works in TABLE \ref{table}. The works that have developed the control strategy to track the trajectory planned by the path planner are mentioned in column `Control strategy'. Among these works, our framework did not require any retraining of RL policies when STL tasks and environment were changed and we were able to reuse the developed RL policies for a particular robot. 

\begin{table}[]
\caption{Comparison with existing approaches for STL tasks}
\label{table}
\resizebox{\columnwidth}{!}{%
\begin{tabular}{@{}lccc@{}}
\toprule
\textbf{Framework} &
  \textbf{\begin{tabular}[c]{@{}l@{}}Control \\ Strategy\end{tabular}} &
  \textbf{\begin{tabular}[c]{@{}l@{}}Adaptability of Policies \\to Different Environments\end{tabular}} &
  \textbf{\begin{tabular}[c]{@{}l@{}}Requirement of \\ Mathematical Model\end{tabular}} \\ \toprule
\begin{tabular}[c]{@{}l@{}}RRT*-based\\ Algorithms \cite{s5}\end{tabular} & \cmark &-  & \cmark \\ \midrule
\begin{tabular}[c]{@{}l@{}}Real-time \\ RRT* \cite{s9}\end{tabular}       & \xmark & - & - \\ \midrule
\begin{tabular}[c]{@{}l@{}}End-to-End \\ RL \cite{s18, s19}\end{tabular}        & \cmark &\xmark  & \xmark \\ \midrule
\begin{tabular}[c]{@{}l@{}}Proposed \\ Approach\end{tabular}    & \cmark & \cmark & \xmark \\ \bottomrule
\end{tabular}%
}
\end{table}

\vspace{-0.2em}
Despite the framework's potential to solve many real-world problems, the approach becomes computationally expensive with the increase in the complexity of the specification and state space, limiting the current work to static environments. In future work, we will consider STL-compliant motion planning problems in the presence of dynamic obstacles and for multi-agent systems.
\bibliography{reference.bib}

\begin{thebibliography}{10}
\providecommand{\url}[1]{#1}
\csname url@samestyle\endcsname
\providecommand{\newblock}{\relax}
\providecommand{\bibinfo}[2]{#2}
\providecommand{\BIBentrySTDinterwordspacing}{\spaceskip=0pt\relax}
\providecommand{\BIBentryALTinterwordstretchfactor}{4}
\providecommand{\BIBentryALTinterwordspacing}{\spaceskip=\fontdimen2\font plus
\BIBentryALTinterwordstretchfactor\fontdimen3\font minus \fontdimen4\font\relax}
\providecommand{\BIBforeignlanguage}[2]{{%
\expandafter\ifx\csname l@#1\endcsname\relax
\typeout{** WARNING: IEEEtran.bst: No hyphenation pattern has been}%
\typeout{** loaded for the language `#1'. Using the pattern for}%
\typeout{** the default language instead.}%
\else
\language=\csname l@#1\endcsname
\fi
#2}}
\providecommand{\BIBdecl}{\relax}
\BIBdecl

\bibitem{s1}
O.~Maler and D.~Nickovic, ``Monitoring temporal properties of continuous signals,'' in \emph{International Symposium on Formal Techniques in Real-Time and Fault-Tolerant Systems}.\hskip 1em plus 0.5em minus 0.4em\relax Springer, 2004, pp. 152--166.

\bibitem{s2}
A.~Donz{\'e} and O.~Maler, ``Robust satisfaction of temporal logic over real-valued signals,'' in \emph{International Conference on Formal Modeling and Analysis of Timed Systems}.\hskip 1em plus 0.5em minus 0.4em\relax Springer, 2010, pp. 92--106.

\bibitem{s4}
P.~Varnai and D.~V. Dimarogonas, ``On robustness metrics for learning {STL} tasks,'' in \emph{American Control Conference}, 2020, pp. 5394--5399.

\bibitem{s5}
C.-I. Vasile, V.~Raman, and S.~Karaman, ``Sampling-based synthesis of maximally-satisfying controllers for temporal logic specifications,'' in \emph{IEEE/RSJ International Conference on Intelligent Robots and Systems (IROS)}, 2017, pp. 3840--3847.

\bibitem{s6}
J.~V. Deshmukh, A.~Donz{\'e}, S.~Ghosh, X.~Jin, G.~Juniwal, and S.~A. Seshia, ``Robust online monitoring of signal temporal logic,'' \emph{Formal Methods in System Design}, vol.~51, pp. 5--30, 2017.

\bibitem{s7}
J.~Karlsson, F.~S. Barbosa, and J.~Tumova, ``Sampling-based motion planning with temporal logic missions and spatial preferences,'' \emph{IFAC-PapersOnLine}, vol.~53, no.~2, pp. 15\,537--15\,543, 2020.

\bibitem{s9}
A.~Linard, I.~Torre, E.~Bartoli, A.~Sleat, I.~Leite, and J.~Tumova, ``{Real-Time RRT}* with signal temporal logic preferences,'' in \emph{IEEE/RSJ International Conference on Intelligent Robots and Systems (IROS)}, 2023, pp. 8621--8627.

\bibitem{s10}
A.~Jones, K.~Leahy, C.~Vasile, S.~Sadradinni, Z.~Serlin, R.~Tron, and C.~Belta, ``Scalable and robust deployment of heterogenenous teams from temporal logic specifications,'' in \emph{International Symposium on Robotics Research (ISRR), Hanoi, Vietnam}, 2019.

\bibitem{s11}
Z.~Liu, B.~Wu, J.~Dai, and H.~Lin, ``Distributed communication-aware motion planning for networked mobile robots under formal specifications,'' \emph{IEEE Transactions on Control of Network Systems}, vol.~7, no.~4, pp. 1801--1811, 2020.

\bibitem{s13}
R.~S. Sutton, A.~G. Barto \emph{et~al.}, ``Reinforcement learning,'' \emph{Journal of Cognitive Neuroscience}, vol.~11, no.~1, pp. 126--134, 1999.

\bibitem{s14}
H.~Sun, W.~Zhang, R.~Yu, and Y.~Zhang, ``Motion planning for mobile robots—focusing on deep reinforcement learning: A systematic review,'' \emph{IEEE Access}, vol.~9, pp. 69\,061--69\,081, 2021.

\bibitem{s15}
H.-T.~L. Chiang, J.~Hsu, M.~Fiser, L.~Tapia, and A.~Faust, ``{RL-RRT}: Kinodynamic motion planning via learning reachability estimators from {RL} policies,'' \emph{IEEE Robotics and Automation Letters}, vol.~4, no.~4, pp. 4298--4305, 2019.

\bibitem{s16}
A.~Francis, A.~Faust, H.-T.~L. Chiang, J.~Hsu, J.~C. Kew, M.~Fiser, and T.-W.~E. Lee, ``Long-range indoor navigation with {PRM-RL},'' \emph{IEEE Transactions on Robotics}, vol.~36, no.~4, pp. 1115--1134, 2020.

\bibitem{s20}
B.~Bischoff, D.~Nguyen-Tuong, I.~Lee, F.~Streichert, A.~Knoll \emph{et~al.}, ``Hierarchical reinforcement learning for robot navigation,'' in \emph{Proceedings of The European Symposium on Artificial Neural Networks, Computational Intelligence And Machine Learning}, 2013.

\bibitem{s21}
S.~LaValle, ``Planning algorithms,'' \emph{Cambridge University Press}, vol.~2, pp. 3671--3678, 2006.

\bibitem{s22}
M.~Pivtoraiko and A.~Kelly, ``Kinodynamic motion planning with state lattice motion primitives,'' in \emph{IEEE/RSJ International Conference on Intelligent Robots and Systems}, 2011, pp. 2172--2179.

\bibitem{s24}
B.~J. Cohen, S.~Chitta, and M.~Likhachev, ``Search-based planning for manipulation with motion primitives,'' in \emph{IEEE International Conference on Robotics and Automation}, 2010, pp. 2902--2908.

\bibitem{s18}
D.~Aksaray, A.~Jones, Z.~Kong, M.~Schwager, and C.~Belta, ``Q-learning for robust satisfaction of signal temporal logic specifications,'' in \emph{IEEE 55th Conference on Decision and Control}, 2016, pp. 6565--6570.

\bibitem{s19}
N.~Saxena, S.~Gorantla, and P.~Jagtap, ``Funnel-based reward shaping for signal temporal logic tasks in reinforcement learning,'' \emph{IEEE Robotics and Automation Letters}, vol.~9, no.~2, pp. 1373--1379, 2023.

\bibitem{s33}
A.~Caballero and G.~Silano, ``A signal temporal logic motion planner for bird diverter installation tasks with multi-robot aerial systems,'' \emph{IEEE Access}, 2023.

\bibitem{s36}
H.~Wang, H.~He, W.~Shang, and Z.~Kan, ``Temporal logic guided motion primitives for complex manipulation tasks with user preferences,'' in \emph{International Conference on Robotics and Automation (ICRA)}, 2022, pp. 4305--4311.

\bibitem{s37}
T.~L{\"o}w, T.~Bandyopadhyay, J.~Williams, and P.~V. Borges, ``Prompt: Probabilistic motion primitives based trajectory planning.'' in \emph{Robotics: Science and Systems}, vol.~17, 2021.

\bibitem{sac1}
T.~Haarnoja, A.~Zhou, P.~Abbeel, and S.~Levine, ``Soft actor-critic: Off-policy maximum entropy deep reinforcement learning with a stochastic actor,'' in \emph{Proceedings of the 35th International Conference on Machine Learning}, vol.~80.\hskip 1em plus 0.5em minus 0.4em\relax PMLR, 2018, pp. 1861--1870.

\bibitem{ppo}
J.~Schulman, F.~Wolski, P.~Dhariwal, A.~Radford, and O.~Klimov, ``Proximal policy optimization algorithms,'' \emph{arXiv preprint arXiv:1707.06347}, 2017.

\bibitem{rrt*}
S.~Karaman and E.~Frazzoli, ``Sampling-based algorithms for optimal motion planning,'' \emph{The International Journal of Robotics Research}, vol.~30, no.~7, pp. 846--894, 2011.

\end{thebibliography}

\end{document}